\documentclass[9pt,conference]{IEEEtran}
\usepackage{amsmath}
\usepackage{amsfonts}
\usepackage{amsthm}
\usepackage{graphicx}
\usepackage{placeins}

\usepackage{algorithm}
\usepackage{algpseudocode}
\usepackage{xcolor}

\newtheorem{definition}{Definition}
\newtheorem{lemma}{Lemma}
\newtheorem{theorem}{Theorem}

\newcommand{\R}{\mathbb{R}_{+}}

\newcommand{\norm}[1]{{\left\| #1\right\|}}

\DeclareMathOperator*{\argmin}{argmin}

\begin{document}

\title{Stratified-NMF for Heterogeneous Data}

\author{
    \IEEEauthorblockN{
    James Chapman\IEEEauthorrefmark{1}\textsuperscript{\textsection}, 
    Yotam Yaniv\IEEEauthorrefmark{1}\textsuperscript{\textsection}, Deanna Needell\IEEEauthorrefmark{1}}
    \IEEEauthorblockA{\IEEEauthorrefmark{1}Department of Mathematics, University of California, Los Angeles}
}

\maketitle
\begingroup\renewcommand\thefootnote{\textsection}
\footnotetext{Equal contribution}
\endgroup

\begin{abstract}
Non-negative matrix factorization (NMF) is an important technique for obtaining low dimensional representations of datasets. However, classical NMF does not take into account data that is collected at different times or in different locations, which may exhibit heterogeneity. We resolve this problem by solving a modified NMF objective, Stratified-NMF, that simultaneously learns strata-dependent statistics and a shared topics matrix. We develop multiplicative update rules for this novel objective and prove convergence of the objective. Then, we experiment on synthetic data to demonstrate the efficiency and accuracy of the method. Lastly, we apply our method to three real world datasets and empirically investigate their learned features. 
\footnote{Preprint version. This article will appear in IEEE Asilomar Conference on Signals, Systems, and Computers 2023.}

\end{abstract}

\section{Introduction}
Non-negative matrix factorization (NMF) is a classical unsupervised machine learning method used in dimensionality reduction and topic modeling \cite{kuang2015nonnegative, wang2012nonnegative, shahnaz2006document}. NMF best addresses data which could be grouped into topics. For example, one could break down a collection of books by topics and then further break down the topics by words associated to them \cite{griffiths2007topics, wallach2006topic}. In a more general setting, we refer to books, topics, and words as \textit{samples}, \textit{features}, and \textit{variables}, respectively. This is modeled by 
$$\argmin_{W, H} ||A-WH||_F^2,\ W\geq 0, H\geq 0$$
where $||\cdot||_F$ denotes the Frobenius norm, $A$ is the data matrix of samples by variables, $H$ is the topics matrix which associates topics to variables, and $W$ is a matrix which associates samples with topics. These names come from the fact that each sample is approximated by a linear combination of the rows in $H$ with coefficients from the rows of $W$. The foundational work by Seung and Lee showed that the standard NMF objective can be optimized efficiently with a multiplicative update \cite{seung2001algorithms, lee1999learning}. Additionally, $W$ and $H$ are chosen to be low rank so that the method learns to compress the data into features. The low rank, efficient multiplicative updates allow this method to scale to large datasets. Additionally, NMF is sought after for its sparse and interpretable feature learning \cite{gillis2014and}. 

A potential drawback of the standard NMF model is that it does not directly account for stratified data, e.g. data drawn from multiple sources \cite{albadawy2018deep}. Data collection methods, along with geographic or time differences, can introduce heterogeneity into the dataset. One common solution is to stratify data in order to obtain accurate results for subgroups of the data \cite{hansen1953sample, nazabal2020handling, roetzheim2012analysis}. In this paper, we augment the NMF objective to account for stratified data.

We consider the setting where multiple groups, or strata, share a common topics dictionary with positive, strata dependent shifts. In this work, strata will be given in the problem, but can be obtained via meta-data or other common, external information. For example, articles written in various regions or across time periods differ in dialect or semantic progression. Although NMF may attribute the words ``pop'' and ``soda'' to the topic ``beverage'', it will be unable to explain the regional differences in the use of this word. Using Stratified-NMF, we can learn the strata dependent shifts and obtain a topic dictionary which measures the underlying commonalities between strata.

In this work, we develop a novel extension of NMF called Stratified-NMF, which is able to account for heterogeneous data. We also derive multiplicative updates and show that the Stratified-NMF objective is non-increasing under our multiplicative update rules. Next we analyze our method on four datasets, spanning synthetic data, image data, census data, and natural language text. We experimentally validate the convergence properties of the method and empirically demonstrate the interpretability of Stratified-NMF. Code for our experiments is also publicly available\footnote{Code can be found at https://github.com/chapman20j/Stratified-NMF}.

\section{Proposed Method}
We consider the case where the dataset has $s$ strata, each with a data matrix $A(i)\in \R^{m_i\times n}$ ($\R$ denotes non-negative real numbers). The data is stored row-wise so that $A(i)_{jk}$ denotes the $k$'th attribute of data point $j$ in stratum $i$. Each data point is assumed to be sampled from a distribution 
$$\mathcal{D}_i = \mathcal{D} + z_i$$
where $\mathcal{D}$ is a 
shared, non-negative distribution and $z_i\in \R^n$ is a strata dependent shift. To account for this, we want to perform NMF on the matrix
\begin{equation}
    A(i) - \begin{bmatrix}
        - & z_i^T & - \\
        - & z_i^T & - \\
         & \vdots & \\
        - & z_i^T & - \\
    \end{bmatrix}
    = A(i) - \mathbf{1} z_i^T
\end{equation}
but over all the strata, where $\mathbf{1}_i \in \R^{m_i}$ is the all ones vector. This leads to the following Stratified-NMF objective. 
\begin{align}
  &  \argmin_{v(i), W(i), H} \sum_{i=1}^s ||A(i) - \mathbf{1} v(i)^T - W(i) H||_F^2, \label{eq:loss}\\
  &  v(i)\geq 0,\ W(i)\geq 0,\ H \geq 0,\ \forall 1\leq i\leq s \label{eq:constraints}
\end{align}
Note that the low rank update to each stratum keeps the number of additional parameters in the NMF objective relatively low. This property is appealing since it prevents overfitting to each stratum.

This formulation allows the model to naturally learn a global representation, $H$, of the data and local representations, $v(i)$, of each stratum. This relieves the model of inductive biases which may be introduced by other stratification techniques. One might consider more simple approaches like centering the means of the different strata or unit normalization of individual data points. The mean centering approach fails because the objective is non-linear and won't behave well under arbitrary additive shifts. 
Similarly, the unit normalization approach may flatten out data at different scales.

\begin{algorithm}[ht]
\caption{Stratified-NMF Multiplicative Update Algorithm} \label{alg:mu-SNMF}
\begin{algorithmic}[1]
\State \textbf{Input:} Data $A(i)\in \R^{m_i\times n}$ for $1\leq i\leq s$, number of iterations $N$, number of $v$ updates per iteration $M$.
\State Construct initial $v(i)\in \R^n, W(i)\in \R^{m_i\times r}, H\in \R^{r\times n}$ for $1\leq i\leq s$
\For{$k$ from 0 to $N-1$}
\For{$j$ from 0 to $M-1$}
\State $v(i) \gets v(i) \frac{A(i)^T\mathbf{1}}{ v(i)m_i + H^T W(i)^T \mathbf{1}} \; \forall i$ \label{eq:updateV}
\EndFor
\State  $W(i) \gets W(i) \frac{A(i) H^T }{(W(i) H + \mathbf{1} v(i)^T)H^T} \; \forall i$ \label{eq:updateW}
\State $H \;\;\;\;\, \gets H \frac{\sum_i W(i)^TA(i)}{\sum_i W(i)^T (W(i) H+\mathbf{1}v(i)^T)}$\label{eq:updateH}
\EndFor
\end{algorithmic}
\end{algorithm}

\section{Theoretical Results}
Due to the similar structure of the Stratified-NMF objective and the standard NMF objective, we are able to develop multiplicative update rules described in Algorithm~\ref{alg:mu-SNMF}. Note that these multiplicative update rules satisfy the constraints in Equation~\ref{eq:constraints} throughout the optimization procedure. Additionally, we state and prove Theorem~\ref{thm:convergence} showing that the multiplicative updates satisfy similar convergence properties to those of the standard NMF objective \cite{seung2001algorithms, dempster1977maximum}.

\begin{definition}(From \cite{seung2001algorithms})
$G(h,h^\prime)$ is an auxiliary function for $F(h)$ if the conditions $G(h,h^\prime) \geq F(h), \; \forall h^\prime$ and $G(h,h) = F(h)$.    
\end{definition}

\begin{lemma}(From \cite{seung2001algorithms})
\label{lem:aux_non_increasing}
If $G$ is an auxiliary function, then $F$ is nonincreasing under the update
$$h^{t+1}=\argmin_h G(h, h^t).$$
\end{lemma}

\begin{theorem}(From \cite{seung2001algorithms}) \label{thm:nmf_non_increasing}The Euclidean distance $||A - W H||_F$ is nonincreasing under the update rules
\begin{equation}
    \label{eqn:standard_nmf_mu}
    H_{ij} \gets H_{ij} \frac{(W^T A)_{ij}}{(W^T WH)_{ij}} \quad
    W_{ij} \gets W_{ij} \frac{(A H^T)_{ij}}{(WHH^T)_{ij}}
\end{equation}
\end{theorem}

\begin{theorem}[Convergence of Multiplicative Update Rules]\label{thm:convergence} The Stratified-NMF objective defined in Equation~\ref{eq:loss} is non-increasing under the update rules defined in Lines~\ref{eq:updateV}, ~\ref{eq:updateW}, and ~\ref{eq:updateH}  of Algorithm~\ref{alg:mu-SNMF}.  
\end{theorem}

\begin{proof}
We reformulate the Stratified-NMF objective into a large block formulation resulting in a standard NMF objective with $s$ additional columns in $W$ and $s$ additional rows in $H$:
\begin{align}
\norm{\hat A- \hat{W}\hat{H}}_{F}^2
\end{align} with
\begin{equation}
    \hat{W} =\begin{bmatrix} 
    \mathbf{1}_1 & \mathbf{0}_1 & \dots & \mathbf{0}_1 & W(1) \\
    \mathbf{0}_2 & \mathbf{1}_2 & \dots & \mathbf{0}_2 & W(2) \\
    \vdots & \vdots & \dots & \vdots & \vdots \\
    \mathbf{0}_s & \mathbf{0}_s & \dots & \mathbf{1}_s & W(s) \\
    \end{bmatrix} \in \R^{\left(\sum_{i=1}^s m_i \right)\times (r + s)},
\end{equation} where $\mathbf{0}_i \in \R^{m_i}$ denotes the zero vector and $\mathbf{1}_i \in \R^{m_i}$ denotes the all-ones vector, and  

\begin{equation}
    \hat A = 
    \begin{bmatrix}
        A(1) \\
        A(2) \\
        \vdots \\
        A(s)
    \end{bmatrix} \in \R^{(\sum_{i=1}^s m_i)\times n}
    ,\ 
    \hat{H} = \begin{bmatrix}
        v(1)^T \\
        v(2)^T \\
        \vdots \\
        v(s)^T \\
        H
    \end{bmatrix} \in \R^{(r+s) \times n}.
\end{equation}
Note that the block matrix $\hat H$ contains parameters, all of which are optimized. Lee and Seung show the non-increasing property of the NMF objective for the $W$ and $H$ multiplicative updates separately \cite{seung2001algorithms}. Therefore to show convergence on the $v$ and $H$, it is enough to show that we obtain the same update formulas from performing NMF on the block formulation. In particular, the standard NMF update says
$$\hat H \gets \hat H \frac{\hat W^T \hat A}{\hat W ^ T \hat W \hat H}.$$
Applying this update rule gives
\begin{align*}
    v(i)^T &= \hat H_i\\
    &\gets v(i)^T \frac{(0, ..., 0, \mathbf{1}^T(i), 0, ..., 0) \hat A}{(0, ..., 0, \mathbf{1}^T(i), 0, ..., 0) \hat W \hat H} \\
    &= v(i)^T \frac{\mathbf{1}^T(i)A(i)}{(0, \cdots, 0, \mathbf{1}^T(i)\mathbf{1}(i), 0, \cdots, 0, \mathbf{1}^T(i) W(i)) \hat H}\\
    &= v(i)^T \frac{\mathbf{1}^T(i)A(i)}{m_i v(i)^T + \mathbf{1}^T(i) W(i)H}.
\end{align*}
Taking transposes on both sides recovers the Stratified-NMF update rule. Let $[H]$ denote the rows corresponding to $H$ in $\hat H$. Then
\begin{align*}
    H &\gets H \frac{\hat W^T_{[H]} \hat A}{\hat W^T_{[H]}\hat W \hat H}\\
    &= H \frac{(W(1)^T,\cdots,W(s)^T) \hat A}{(W(1)^T,\cdots,W(s)^T) \hat W \hat H}\\
    &= H \frac{\sum_i W(i)^T A(i)}{(W(1)^T\mathbf{1}(1),\cdots, W(s)^T \mathbf{1}(s), \sum_i W(i)^T W(i)) \hat H}\\
    &= H \frac{\sum_i W(i)^T A(i)}{\sum_i W(i)^T\mathbf{1}(i)v(i)^T + W(i)^T W(i) H }.
\end{align*}
By Theorem~\ref{thm:nmf_non_increasing}, the Stratified-NMF objective is non-increasing under the multiplicative updates defined in Equation~\ref{eq:updateV} and Equation~\ref{eq:updateH}. 

Now we look at the multiplicative update for $W$. Following the proof technique in \cite{seung2001algorithms}, consider the function
$$F_{\sigma, r}(w) = \frac{1}{2} \sum_c \left(a_c - v_c - \sum_j w_j(\sigma)H_{j, c}\right)^2,$$
where $w_j = W_{rj}(\sigma)$, $a_c = A(\sigma)_{rc}$, $v_c = v(c)$. We define $K$ to be the diagonal matrix with entries
$$K_{aa} = \frac{(wHH^T + v^T H^T)_a}{w_a} = \frac{(wHH^T)_a}{w_a} + \frac{(v^T H^T)_a}{w_a} = K_{aa}^1 + K_{aa}^2.$$
We also define
$$G_{\sigma, r}(w, w^t) = F_{\sigma, r}(w^t) + (w-w^t)\nabla F_{\sigma, r}(w^t) + \frac{1}{2} (w-w^t) K(w^t) (w-w^t)$$
and expand $F_{\sigma, r}$ to obtain
$$F_{\sigma, r}(w) = F_{\sigma, r}(w^t) + (w-w^t)\nabla F_{\sigma, r}(w^t) + \frac{1}{2} (w-w^t) HH^T (w-w^t).$$
To show that $G_{\sigma, r}$ is an auxiliary function of $F_{\sigma, r}$, it suffices to show that 
$$K - HH^T$$
is positive semi-definite. We see that 
$$K^1 +K^2 - HH^T \succeq K^1 - HH^T \succeq 0,$$
where the first inequality comes from the fact that $K^2$ is non-negative and diagonal. The second inequality was shown in the proof of Lemma~2 from \cite{seung2001algorithms} as $K^1$ is the same as $K$ in their proof. Summing over the strata and rows, we obtain 
$$G_{\sigma, r} \geq F_{\sigma, r},\ \forall \sigma, r \implies G = \sum_{\sigma, r} G_{s, r} \geq \sum_{\sigma, r} F_{\sigma, r} = F,$$
where $F$ is the Stratified-NMF objective and minimizing $G$ recovers the multiplicative update for $W$ defined in Equation~\ref{eq:updateW}. In other words, $G$ is an auxiliary function for $F$. By Lemma~\ref{lem:aux_non_increasing}, the Stratified-NMF objective is non-increasing under the multiplicative update rule defined in Equation~\ref{eq:updateW}.
\end{proof}

\section{Experimental Results}
Next we experimentally analyze the performance of Stratified-NMF on synthetic, housing, image and text datasets. We show the merits of Stratified-NMF as a tool for interpretable unsupervised learning in these unsupervised learning applications. In all experiments $H, \; W(i)\; \forall i$ are initialized as random uniform in $[0,1/\sqrt{r}\,]$ entries, independently and identically distributed (iid) for all parameters. The $v(i) \; \forall i$ are initialized as random uniform $[0,1]$ iid entries for all parameters. The number of updates for $v$ at each iteration (denoted $M$ in Algorithm~\ref{alg:mu-SNMF}) is set to 2 for all experiments. We add $10^{-9}$ to all the entries in the denominator of each update to avoid division by zero. In this work, \textit{loss} denotes the square root of the Stratified-NMF objective and \textit{normalized loss} refers to the following quantity:
$$\sqrt{\frac{\sum_{i=1}^s ||A(i) - \mathbf{1} v(i)^T - W(i) H||_F^2}{\sum_{i=1}^s ||A(i)||_F^2}}.$$

\subsection{Synthetic}
We test the Stratified-NMF objective on a synthetic dataset containing $4$ strata, $A(1),\,A(2),\,A(3),\,A(4)$. Each matrix, $A(i) \in \R^{100 \times 100}$ is constructed by multiplying a $100 \times 5$ matrix with a $5\times 100$ matrix each of which have entries that are drawn uniformly from $[0,1]$. We then add a synthetic strata feature of $\mathbf{1}v_{\text{true}}^T(i)$ where $v_{\text{true}}(i)$ has entries drawn uniformly from $[i-1,i]$. We run stratified-NMF for $10000$ iterations and analyze the results. 

Figure~\ref{fig:loss} shows the log-log plot of the normalized stratified NMF objective versus iterations. Observe that the multiplicative update detailed in Algorithm~\ref{alg:mu-SNMF} converges on the synthetic dataset in under $2000$ iterations and obtains a final normalized loss of $9.7e-4$. 
Additionally, in Figure~\ref{fig:strata_mean}, we measure the success of recovering the learned strata features by plotting the mean value of each $v(i)$ per iteration. Since $v_{\text{true}}^T(i)$ has entries drawn uniformly from $[i-1,i]$ we expect the $v(i)$'s to converge to $i-0.5$, which corresponds to $0.5$ for $v(1)$, $1.5$ for $v(2)$, $2.5$ for $v(3)$, and $3.5$ for $v(4)$. Figure~\ref{fig:strata_mean} shows that the means converge to $0.51,\, 1.47,\, 2.53,\, 3.57$, respectively. These means are very close to the expected means, which highlights the utility of the $v(i)$'s in recovering information about the individual strata. 

\begin{figure}[ht]
    \centering
    \includegraphics[scale=.58]{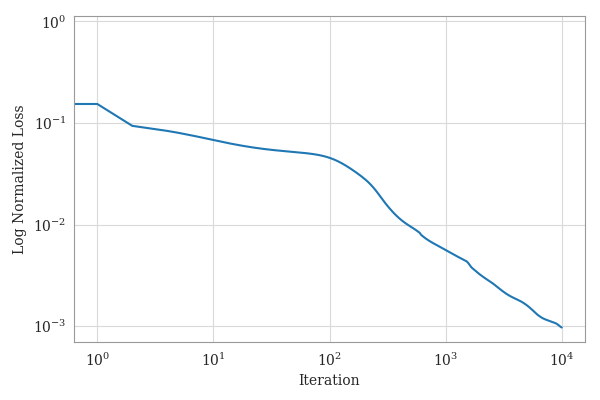}
    \caption{Log of normalized loss versus number of iterations for the synthetic experiment. The final normalized loss is $9.7e-4$. }
    \label{fig:loss}
\end{figure}

\begin{figure}[ht]
    \centering
    \includegraphics[scale=.58]{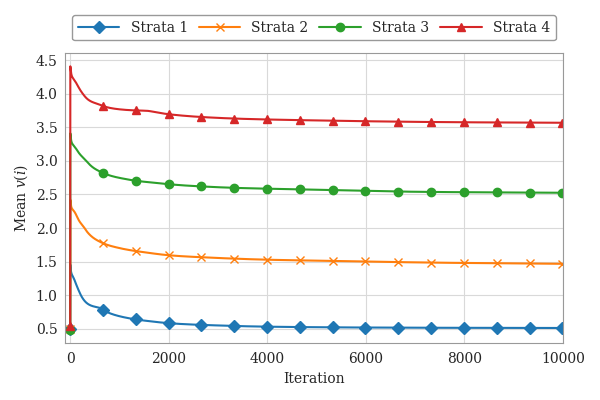}
    \caption{Means of each strata feature $v(i)$ over the four strata of the synthetic experiment. The final means are $0.51,\, 1.47,\, 2.53,\, 3.57$, respectively. We expect to approximately recover the true means of $0.5,\, 1.5,\, 2.5,\, 3.5$. 
    }
    \label{fig:strata_mean}
\end{figure}

\subsection{California Housing}
We use Stratified-NMF to analyze the California housing dataset originally compiled in 1997 available on scikit-learn \cite{pace1997sparse, scikit-learn}. This dataset was curated from the 1990 US census data and contains average income, housing average age, average rooms, average bedrooms, population, average occupation, latitude, and longitude fields. We stratify the data by splitting it into low, medium and high income groups and drop the latitude and longitude fields due to the positivity constraint. The stratified dataset has $815$ samples for the low income stratum all of which make less than $15k$ dollar household income per year, $1492$ samples for the medium income stratum all of which make between $45k-50k$ household income per year, and $308$ samples for the high income stratum which make more than $100k$ per year. Although the sizes of the stratum are imbalanced, Stratified-NMF is able to capture meaningful local features. We run Stratified-NMF on the three strata for $100$ iterations with a rank of $5$. We then analyze the interpretability of the learned strata features, $v(i)$, in Figure~\ref{fig:cali_means}.

Figure~\ref{fig:cali_means} shows the normalized strata features $v(i)$ for the California dataset. This plot was obtained by taking the resulting $v(i)$'s and normalizing them so that $\sum_j v(i)_j = 1$ for all $j$. This depicts the relative differences between strata on the same plot. The median income variable shows an obvious trend that low, medium, and high income is preserved. Interestingly, house age and average bedrooms are relatively constant across strata, while the average rooms seems to increase with wealth. Lastly, the population is highest for the medium income stratum. An analysis by the Brookings Institute reported a similar trend found in survey data from 2017 \cite{berube2018does}. Note that average occupancy was removed from the plot because it was near zero across all strata and the normalization causes one stratum to visually dominate this variable.

\begin{figure}[ht]
    \centering
    \includegraphics[scale=.58]{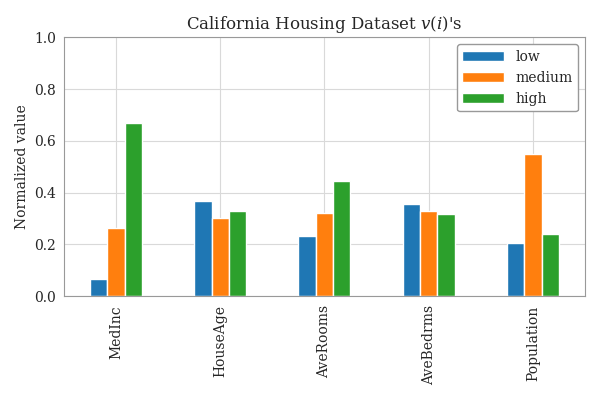}
    \caption{Normalized strata features $v(i)$ for the California dataset. This plot was obtained by taking the resulting $v(i)$'s and normalizing them so that $\sum_j v(i)_j = 1$ for all $j$. The median income and average rooms increase with income, as expected. Interestingly, median income individuals tend to live in more populous regions \cite{berube2018does}. Trends for the other variables are neutral.
    }
    \label{fig:cali_means}
\end{figure}

\subsection{MNIST}
We also test our method on MNIST, a dataset of handwritten digits sourced from the Torchvision library \cite{lecun2010mnist, marcel2010torchvision}. We stratify the data into two strata $S_1, S_2$. The $S_1$ stratum consists of $100$ images of $1$'s and $100$ images of $2$'s, while the $S_2$ stratum consists of $100$ images of $2$'s and $100$ images of $3$'s. Note that the samples in $S_1$ and $S_2$ are disjoint. We expect $H$ to capture the $2$'s as global features since they are in both strata while we expect the $v$'s to capture the $1$'s and $3$'s which are unique to their respective strata. Stratified-NMF is run on the dataset for $100$ iterations with a rank of $5$. The learned strata features, $v(i)$, are displayed in Figure~\ref{fig:mnist_v} and the global features, $H$, are displayed in Figure~\ref{fig:mnist_h}.

In Figure~\ref{fig:mnist_v} we observe that $v(1)$ captures the $1$, which is unique to $S_1$, and $v(2)$ captures the $3$, which is unique to $S_2$. Many of the shared features shown in Figure~\ref{fig:mnist_h} resemble $2$'s, which is the common data to both stratum. This suggests that the strata features learn information specific to each stratum while the global dictionary $H$ learns features common to the whole dataset. The remaining feature in $H$ resembles a $3$. We suspect that this is due to the $2$ and $3$ not having enough commonality, which means that $H$ must store features corresponding to $3$'s in order to fit the data. Note also that rows of $H$ are features of the residuals $A(i)-\mathbf{1}v(i)^T$. This can be seen in Figure~\ref{fig:mnist_h} as many of the features resemble numbers with missing lines or faint regions. The missing parts correspond to the parts extracted from $\mathbf{1}v(i)^T$.

\begin{figure}[ht]
    \centering
        \includegraphics[width=0.32\linewidth]{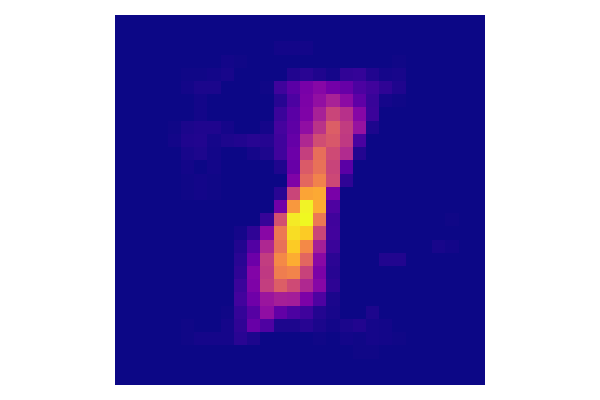}
        \includegraphics[width=0.32\linewidth]{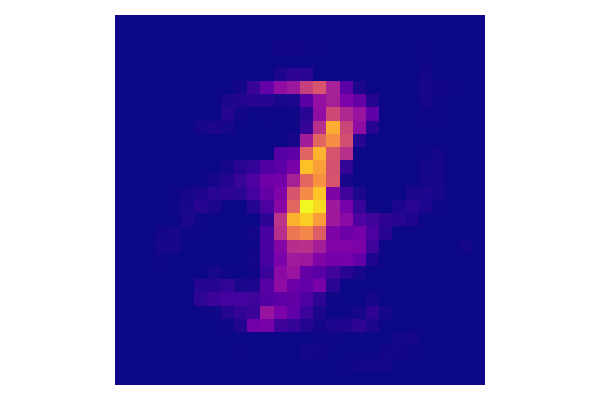}
      \caption{Learned strata features on the MNIST experiment. The left image is a plot of $v(1)$ and it resembles a one. The right image is a plot of $v(2)$ and it resembles parts of a three.}
    \label{fig:mnist_v}
\end{figure}

\begin{figure}[ht]
        \centering
        \includegraphics[width=0.32\linewidth]{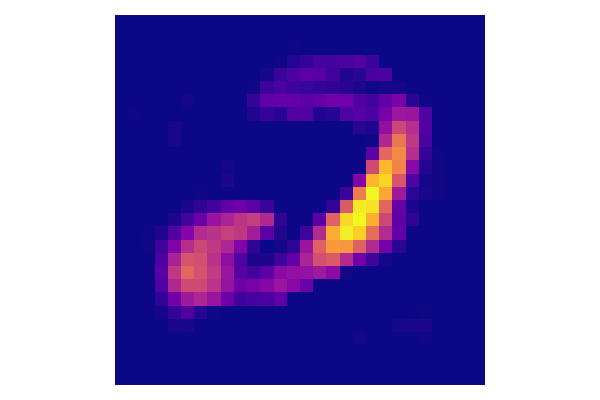}
        \includegraphics[width=0.32\linewidth]{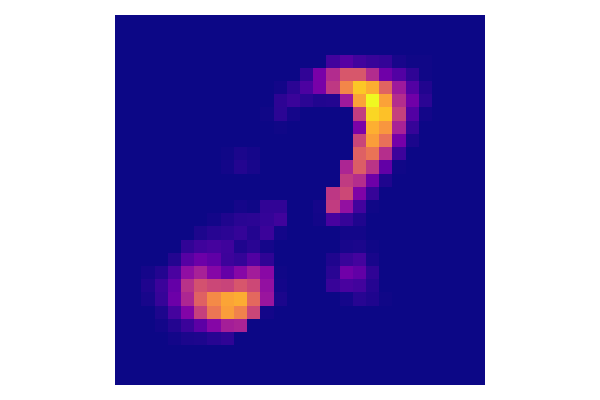}
        \includegraphics[width=0.32\linewidth]{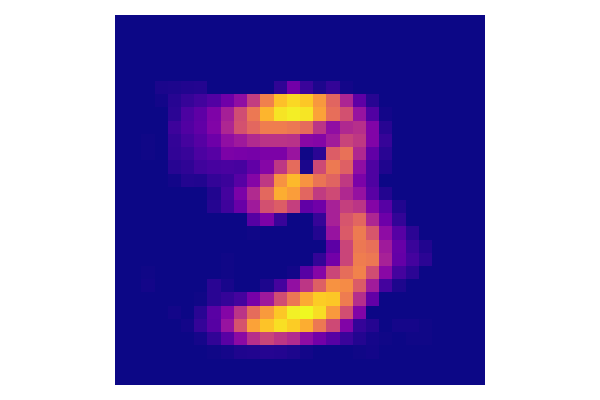}
        \includegraphics[width=0.32\linewidth]{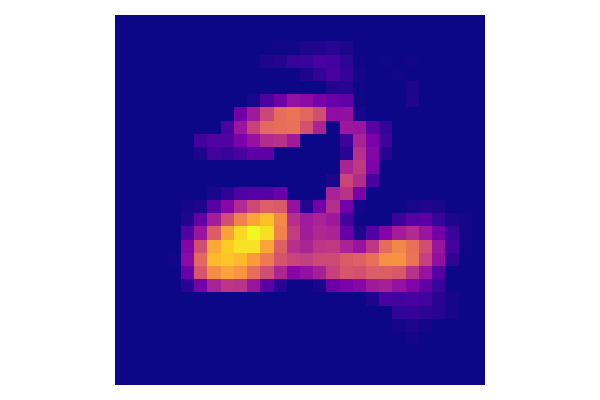}
        \includegraphics[width=0.32\linewidth]{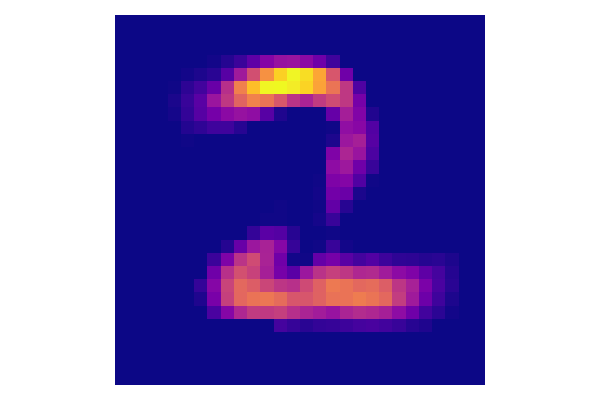}
    \caption{Learned topics matrix for the MNIST experiment. Each image is a row of $H$. These images capture global features across the residuals $A(i)-\mathbf{1}v(i)^T$. Four of the five learned global features are $2$'s with missing segments which appear to be captured by the $v(i)$'s. The remaining feature in $H$ resembles a $3$.}
    \label{fig:mnist_h}
\end{figure}

\subsection{20 newsgroups dataset}
Finally, we test our method on the 20 newsgroups dataset sourced from scikit-learn \cite{scikit-learn, sklearn20newsgroups}. This dataset is composed of 20 categories of news articles encompassing a variety of topics including hockey, cryptography, space, and medicine. Each category contains between 799 and 999 articles from each news group. We preprocess the data by removing all stop words, headers, footers and quotes and subsequently compute a global text frequency inverse document frequency matrix. This yields a length $51840$ sparse non-negative feature vector for each article. We stratify the data according to the news group and run Stratified-NMF for $100$ iterations with a rank of $20$. 

Figure~\ref{tab:newsgroup_words} displays the top three, highest weighted, words captured by the strata features for each category. We observe that for each of the twenty stratum, the top words captured by the strata features correspond exactly to their topic. This highlights the interpretability of the local features. For example, the top three words in the autos newsgroup are ``driving'', ``v6'', and ``transmission'', all of which may be used when discussing automobiles. Additionally, the top words in the hockey newsgroup are ``Toronto'', ``sabres'', and ``coach'', all of which are associated with professional hockey teams. Finally, the top words in the medicine newsgroup are ``kidney'', ``symptoms'', and ``doctors'', all of which are associated with medicine. Upon looking at the top thirty words (contained in the Github), we were manually able to discern which category the words corresponded to. Finally, we observed that the local strata features were sparse with approximately $20\%$ nonzero, highlighting sparsity preservation properties of Stratified-NMF.

\begin{figure}[ht]
    \centering
    \begin{tabular}{|l|lll|}
    \hline
    Newsgroup & 1 & 2 & 3 \\
    \hline
    atheism & cheers & exist & bobbe \\
    graphics & convert & points & vesa \\
    misc computer & norton & tried & truetype \\
    pc hardware & work & set & com \\
    mac hardware & clock & upgrade & hardware \\
    windows x & hp & colormap & r5 \\
    forsale & following & looking & manual \\
    autos & driving & v6 & transmission \\
    motorcycles & tony & cop & wheel \\
    baseball & hitting & jays & phillies \\
    hockey & toronto & sabres & coach \\
    cryptography & random & public & voice \\
    electronics & detector & supply & chip \\
    medicine & kidney & symptoms & doctors \\
    space & new & long & astronomy \\
    christian & churches & read & body \\
    politics guns & jim & shot & clinton \\
    politics mideast & west & civilians & countries \\
    politics misc & national & laws & house \\
    religion misc & life & cult & context \\
    \hline
\end{tabular}
    \caption{Top three words learned for each newsgroup in the 20 newsgroups dataset. These are obtained by looking at the largest values in each $v(i)$. We observe that the forsale newsgroup has the words ``following'', ``looking'' and ``manual'', which are associated with selling or shipping items. Similarly, the baseball newsgroup has the words ``hitting'', ``jays'', and ``phillies'', which correspond to baseball or major league baseball teams.
    }
     \label{tab:newsgroup_words}
\end{figure}

\section{Discussion}
NMF is a popular unsupervised learning method with many extensions and variations \cite{gillis2014and}. A very natural question is if one can apply these stratification techniques to existing NMF formulations. We would like to explore extending our method to other NMF variants such as sparsity constrained NMF \cite{kim2008sparse,kassab2023detecting}, semi-supervised NMF \cite{wang2015semi, haddock2020semi, vendrow2021guided}, online NMF \cite{guan2012online}, and tensor NMF \cite{kassab2023detecting, Zafeiriou2009}. Augmenting these methods to include stratification could provide more insight into subgroups within the data. 
There is also additional work to be done on the theoretical side. Work by Lin further analyzes the convergence properties of NMF and proves convergence under slight modifications to the original NMF multiplicative update \cite{lin2007convergence}. Similar results should hold for Stratified-NMF, but this is outside of the scope of this paper. We believe that the following present promising and important future work:
\begin{itemize}
    \item Variable number of learnable parameters for each stratum
    \item More extensive testing on real world datasets
    \item Initialization of Stratified-NMF Matrices \cite{wild2004improving} 
    \item Automatic hyper-parameter tuning \cite{liu2021gptune}
\end{itemize}
Stratification of unsupervised algorithms may yield more interpretable results on smaller subgroups of data, while preserving global trends in the data.

\section{Conclusion}
In this paper we propose Stratified-NMF to extend NMF to capture heterogeneous data. We first propose the Stratified-NMF objective function and derive multiplicative update rules. 
Then, we prove that the Stratified-NMF objective is non-increasing with respect to the derived, multiplicative updates. Finally, we explore the application of our method to unsupervised learning tasks spanning different modalities including synthetic data, census data, image data, and natural language text data. Stratified-NMF is able to simultaneously capture local and global strata information in the dataset. We highlight the interpretability of the local and the global features learned by the method.

\section{Acknowledgements}
JC, DN and YY were supported by NSF DMS 2108479.
YY was also supported by the  University of California, Los Angeles dissertation year fellowship.

\FloatBarrier
\bibliographystyle{IEEEtran}
\bibliography{IEEEabrv,ref}

\begin{thebibliography}{10}
\providecommand{\url}[1]{#1}
\csname url@samestyle\endcsname
\providecommand{\newblock}{\relax}
\providecommand{\bibinfo}[2]{#2}
\providecommand{\BIBentrySTDinterwordspacing}{\spaceskip=0pt\relax}
\providecommand{\BIBentryALTinterwordstretchfactor}{4}
\providecommand{\BIBentryALTinterwordspacing}{\spaceskip=\fontdimen2\font plus
\BIBentryALTinterwordstretchfactor\fontdimen3\font minus \fontdimen4\font\relax}
\providecommand{\BIBforeignlanguage}[2]{{%
\expandafter\ifx\csname l@#1\endcsname\relax
\typeout{** WARNING: IEEEtran.bst: No hyphenation pattern has been}%
\typeout{** loaded for the language `#1'. Using the pattern for}%
\typeout{** the default language instead.}%
\else
\language=\csname l@#1\endcsname
\fi
#2}}
\providecommand{\BIBdecl}{\relax}
\BIBdecl

\bibitem{kuang2015nonnegative}
D.~Kuang, J.~Choo, and H.~Park, ``Nonnegative matrix factorization for interactive topic modeling and document clustering,'' \emph{Partitional clustering algorithms}, pp. 215--243, 2015.

\bibitem{wang2012nonnegative}
Y.-X. Wang and Y.-J. Zhang, ``Nonnegative matrix factorization: A comprehensive review,'' \emph{IEEE Transactions on knowledge and data engineering}, vol.~25, no.~6, pp. 1336--1353, 2012.

\bibitem{shahnaz2006document}
F.~Shahnaz, M.~W. Berry, V.~P. Pauca, and R.~J. Plemmons, ``Document clustering using nonnegative matrix factorization,'' \emph{Information Processing \& Management}, vol.~42, no.~2, pp. 373--386, 2006.

\bibitem{griffiths2007topics}
T.~L. Griffiths, M.~Steyvers, and J.~B. Tenenbaum, ``Topics in semantic representation.'' \emph{Psychological review}, vol. 114, no.~2, p. 211, 2007.

\bibitem{wallach2006topic}
H.~M. Wallach, ``Topic modeling: beyond bag-of-words,'' in \emph{Proceedings of the 23rd international conference on Machine learning}, 2006, pp. 977--984.

\bibitem{seung2001algorithms}
H.~S. Seung and D.~D. Lee, ``Algorithms for non-negative matrix factorization,'' \emph{Advances in neural information processing systems}, vol.~13, pp. 556--562, 2001.

\bibitem{lee1999learning}
D.~D. Lee and H.~S. Seung, ``Learning the parts of objects by non-negative matrix factorization,'' \emph{Nature}, vol. 401, no. 6755, pp. 788--791, 1999.

\bibitem{gillis2014and}
N.~Gillis, ``The why and how of nonnegative matrix factorization,'' \emph{Regularization, optimization, kernels, and support vector machines}, vol.~12, no. 257, pp. 257--291, 2014.

\bibitem{albadawy2018deep}
E.~A. AlBadawy, A.~Saha, and M.~A. Mazurowski, ``Deep learning for segmentation of brain tumors: Impact of cross-institutional training and testing,'' \emph{Medical physics}, vol.~45, no.~3, pp. 1150--1158, 2018.

\bibitem{hansen1953sample}
M.~H. Hansen, W.~N. Hurwitz, and W.~G. Madow, \emph{Sample survey methods and theory. Vol. I. Methods and applications.}\hskip 1em plus 0.5em minus 0.4em\relax John Wiley, 1953.

\bibitem{nazabal2020handling}
A.~Nazabal, P.~M. Olmos, Z.~Ghahramani, and I.~Valera, ``Handling incomplete heterogeneous data using vaes,'' \emph{Pattern Recognition}, vol. 107, p. 107501, 2020.

\bibitem{roetzheim2012analysis}
R.~G. Roetzheim, K.~M. Freund, D.~K. Corle, D.~M. Murray, F.~R. Snyder, A.~C. Kronman, P.~Jean-Pierre, P.~C. Raich, A.~E. Holden, J.~S. Darnell \emph{et~al.}, ``Analysis of combined data from heterogeneous study designs: an applied example from the patient navigation research program,'' \emph{Clinical Trials}, vol.~9, no.~2, pp. 176--187, 2012.

\bibitem{dempster1977maximum}
A.~P. Dempster, N.~M. Laird, and D.~B. Rubin, ``Maximum likelihood from incomplete data via the em algorithm,'' \emph{Journal of the royal statistical society: series B (methodological)}, vol.~39, no.~1, pp. 1--22, 1977.

\bibitem{pace1997sparse}
R.~K. Pace and R.~Barry, ``Sparse spatial autoregressions,'' \emph{Statistics \& Probability Letters}, vol.~33, no.~3, pp. 291--297, 1997.

\bibitem{scikit-learn}
F.~Pedregosa, G.~Varoquaux, A.~Gramfort, V.~Michel, B.~Thirion, O.~Grisel, M.~Blondel, P.~Prettenhofer, R.~Weiss, V.~Dubourg, J.~Vanderplas, A.~Passos, D.~Cournapeau, M.~Brucher, M.~Perrot, and E.~Duchesnay, ``Scikit-learn: Machine learning in {P}ython,'' \emph{Journal of Machine Learning Research}, vol.~12, pp. 2825--2830, 2011.

\bibitem{berube2018does}
A.~Berube, ``Where does the american middle class live,'' \emph{Brookings Institute Report}, 2018.

\bibitem{lecun2010mnist}
Y.~LeCun, C.~Cortes, and C.~Burges, ``Mnist handwritten digit database,'' \emph{ATT Labs [Online]. Available: http://yann.lecun.com/exdb/mnist}, vol.~2, 2010.

\bibitem{marcel2010torchvision}
S.~Marcel and Y.~Rodriguez, ``Torchvision the machine-vision package of torch,'' in \emph{Proceedings of the 18th ACM international conference on Multimedia}, 2010, pp. 1485--1488.

\bibitem{sklearn20newsgroups}
J.~Rennie, ``20 newsgroups text dataset,'' 2008.

\bibitem{kim2008sparse}
J.~Kim and H.~Park, ``Sparse nonnegative matrix factorization for clustering,'' Georgia Institute of Technology, Tech. Rep., 2008.

\bibitem{kassab2023detecting}
L.~Kassab, A.~Kryshchenko, H.~Lyu, D.~Molitor, D.~Needell, E.~Rebrova, and J.~Yuan, ``Sparseness-constrained nonnegative tensor factorization for detecting topics at different time scales,'' \emph{arXiv preprint arXiv:2010.01600}, 2023.

\bibitem{wang2015semi}
D.~Wang, X.~Gao, and X.~Wang, ``Semi-supervised nonnegative matrix factorization via constraint propagation,'' \emph{IEEE transactions on cybernetics}, vol.~46, no.~1, pp. 233--244, 2015.

\bibitem{haddock2020semi}
J.~Haddock, L.~Kassab, S.~Li, A.~Kryshchenko, R.~Grotheer, E.~Sizikova, C.~Wang, T.~Merkh, R.~Madushani, M.~Ahn \emph{et~al.}, ``Semi-supervised nmf models for topic modeling in learning tasks,'' \emph{arXiv preprint arXiv:2010.07956}, 2020.

\bibitem{vendrow2021guided}
J.~Vendrow, J.~Haddock, E.~Rebrova, and D.~Needell, ``On a guided nonnegative matrix factorization,'' in \emph{ICASSP 2021-2021 IEEE International Conference on Acoustics, Speech and Signal Processing (ICASSP)}.\hskip 1em plus 0.5em minus 0.4em\relax IEEE, 2021, pp. 3265--32\,369.

\bibitem{guan2012online}
N.~Guan, D.~Tao, Z.~Luo, and B.~Yuan, ``Online nonnegative matrix factorization with robust stochastic approximation,'' \emph{IEEE Transactions on Neural Networks and Learning Systems}, vol.~23, no.~7, pp. 1087--1099, 2012.

\bibitem{Zafeiriou2009}
S.~Zafeiriou, \emph{Algorithms for Nonnegative Tensor Factorization}.\hskip 1em plus 0.5em minus 0.4em\relax London: Springer London, 2009, pp. 105--124.

\bibitem{lin2007convergence}
C.-J. Lin, ``On the convergence of multiplicative update algorithms for nonnegative matrix factorization,'' \emph{IEEE Transactions on Neural Networks}, vol.~18, no.~6, pp. 1589--1596, 2007.

\bibitem{wild2004improving}
S.~Wild, J.~Curry, and A.~Dougherty, ``Improving non-negative matrix factorizations through structured initialization,'' \emph{Pattern recognition}, vol.~37, no.~11, pp. 2217--2232, 2004.

\bibitem{liu2021gptune}
Y.~Liu, W.~M. Sid-Lakhdar, O.~Marques, X.~Zhu, C.~Meng, J.~W. Demmel, and X.~S. Li, ``Gptune: Multitask learning for autotuning exascale applications,'' in \emph{Proceedings of the 26th ACM SIGPLAN Symposium on Principles and Practice of Parallel Programming}, 2021, pp. 234--246.

\end{thebibliography}

\end{document}